\documentclass[letterpaper, 10 pt, conference]{ieeeconf} 
\IEEEoverridecommandlockouts

\usepackage{times}
\usepackage{color, colortbl, xcolor}

\newcommand{\tvar}{t}
\newcommand{\tdummy}{\tau}
\newcommand{\R}{\mathbb{R}}
\newcommand{\ctrl}{u}

\newcommand{\cfunc}{u(\cdot)}

\newcommand{\cset}{\mathcal{U}}

\newcommand{\state}{x}
\newcommand{\traj}{\xi} 

\newcommand{\dyn}{f} 
\newcommand{\targetfunc}{l}
\newcommand{\targetset}{\mathcal{L}}
\newcommand{\safeset}{\mathcal{S}}

\newcommand{\BRT}{\text{BRT}}
\newcommand{\policy}{\pi}

\newcommand{\costfunctional}{J}

\newcommand{\vfunc}{V}

\newcommand{\trajstandard}{\traj_{\state,\tvar}^{\ctrl}}
\newcommand{\trajinittime}{\traj_{\state,0}^{\ctrl}}

\newcommand{\veh}{Q}

\newcommand{\horizon}{T}

\newcommand{\metName}{\text{DeepReach}}
\newcommand{\approachName}{\text{Scenario Optimization Verification}}

\newcommand{\safetyMetric}{\delta_{\Tilde{\vfunc},\Tilde{\policy}}}
\newcommand{\costFunction}{J_{\Tilde{\policy}}}

\newtheorem{remark}{Remark}

\newtheorem{theorem}{Theorem}

\newtheorem{lemma}{Lemma}

\usepackage{multicol}
\usepackage[bookmarks=true]{hyperref}

\usepackage{amsfonts}
\usepackage{algorithmic}
\usepackage[ruled,vlined,titlenotnumbered,linesnumbered]{algorithm2e} 
\usepackage{amsmath}
\usepackage{dsfont}
\usepackage{multicol}
\usepackage[pdftex]{graphicx}   
\usepackage{graphicx}
\usepackage{bbm}
\usepackage{mathtools}
\usepackage{subfig}
\usepackage{dsfont}
\usepackage{cite}
\usepackage{accents}

\graphicspath{{./figs/}}    
\DeclareGraphicsExtensions{.pdf,.png}  

\usepackage[textfont=md,font=footnotesize]{caption}

\usepackage[left=55pt, right=55pt, top=55pt, bottom=58pt]{geometry} 

\IEEEoverridecommandlockouts
\usepackage{balance}

\usepackage{lipsum}

\definecolor{planning_color}{RGB}{69, 174, 254}    
\definecolor{prediction_color}{RGB}{255, 116, 190}
    
\newcommand{\example}[1]%
{
\textbf{Running example:}
\textit{#1}
}

\begin{document}
\title{Generating Formal Safety Assurances for High-Dimensional Reachability}
\author{Albert Lin$^{1}$ and Somil Bansal$^{2}$
\thanks{$^{1}$Author is with CS at Princeton: \{aklin\}@princeton.edu.
$^{2}$Author is with ECE at University of Southern California: \{somilban\}@usc.edu.
Project Website: \url{https://sia-lab-git.github.io/DeepReach_Verification/}.
This research is supported in part by the NVIDIA Academic Hardware Grant Program and the NSF REU Program at USC in robotics and autonomous systems.
}
}
%
%
\maketitle
%

\begin{abstract}
Providing formal safety and performance guarantees for autonomous systems is becoming increasingly important. Hamilton-Jacobi (HJ) reachability analysis is a popular formal verification tool for providing these guarantees, since it can handle general nonlinear system dynamics, bounded adversarial system disturbances, and state and input constraints. 
However, it involves solving a PDE, whose computational and memory complexity scales exponentially with respect to the state dimensionality, making its direct use on large-scale systems intractable.
A recently proposed method called \metName{} overcomes this challenge by leveraging a sinusoidal neural PDE solver for high-dimensional reachability problems, whose computational requirements scale with the complexity of the underlying reachable tube rather than the state space dimension.
Unfortunately, neural networks can make errors and thus the computed solution may not be safe, which falls short of achieving our overarching goal to provide formal safety assurances. 
In this work, we propose a method to compute an error bound for the DeepReach solution. 
This error bound can then be used for reachable tube correction, resulting in a safe approximation of the true reachable tube.
We also propose a scenario-based optimization approach to compute a probabilistic bound on this error correction for general nonlinear dynamical systems.
We demonstrate the efficacy of the proposed approach in obtaining probabilistically safe reachable tubes for high-dimensional rocket-landing and multi-vehicle collision-avoidance problems.
\end{abstract}

\IEEEpeerreviewmaketitle

\section{Introduction}
It is becoming increasingly important that we can design provably safe controllers for autonomous systems.
Hamilton-Jacobi (HJ) reachability analysis provides a powerful framework to design such controllers for general nonlinear dynamical systems \cite{lygeros2004reachability, mitchell2005time}.
In reachability analysis, the safe states for the system are characterized through the \textit{Backward Reachable Tube (BRT)} of the system.
This is the set of states from which trajectories will eventually reach some given target set despite the best control effort.
Thus, if the target set represents the set of undesirable states, the BRT represents unsafe states for the system and should be avoided. 
Along with the BRT, reachability analysis also provides a safety controller to keep the system outside the BRT. 

Traditionally, the BRT computation in HJ reachability is formulated as an optimal control problem.
The BRT can then be obtained as a sub-zero level solution of the corresponding value function.
Obtaining the value function requires solving a partial differential equation (PDE) over a state-space grid, resulting in an exponentially scaling computation complexity with the number of states \cite{bansal2017hamilton}.
To overcome this challenge, a variety of solutions have been proposed that trade off between the class of dynamics
they can handle, the approximation quality of the BRT,
and the required computation.
These include specialized methods for linear and affine dynamics \cite{10.1007/3-540-64358-3_38, Frehse2011, Kurzhanski00, Kurzhanski02, Maidens13, girard2005reachability, althoff2010computing, bak2019numerical, Nilsson2016}, polynominal dynamics \cite{doi:10.1177/0278364914528059, majumdar2017funnel, Dreossi16, henrion2014convex}, monotonic dynamics \cite{coogan2015efficient}, and convex dynamics \cite{chow2017algorithm}
(see \cite{bansal2017hamilton, bansal2021deepreach} for a survey).
Owing to the success of deep learning, there has also been a surge of interest in approximating high-dimensional BRTs \cite{rubies2019classification, fisac2019bridging, djeridane2006neural, niarchos2006neural, darbon2020overcoming} and optimal controllers \cite{onken2022neural} through deep neural networks (DNN).
Building upon this line of work, the authors in \cite{bansal2021deepreach} have proposed DeepReach -- a toolbox that leverages recent advances in neural implicit representations and neural PDE solvers to compute a value function and a safety controller for high-dimensional systems.
Compared to the aforementioned methods, DeepReach can handle general nonlinear dynamics, the presence of exogenous disturbances, as well as state and input constraints during the BRT computation.
However, even if we can now obtain an approximate BRT for high-dimensional systems, it is limited in usefulness by its approximate nature. 
In particular, the learned BRT can be overly optimistic; thus, no formal safety guarantees can be provided on the obtained BRT and controller. 

In this work, our goal is to compute a BRT and a safety controller for high-dimensional systems with provable safety guarantees.
Specifically, we build upon DeepReach and propose a verification method to provide safety assurances on the DeepReach solution.
The key insight of our work is to rely on the consistency between the learned value function and the implicit safety controller induced by this value function to compute a \textit{uniform correction bound} for the value function.
Essentially, if a particular state is outside the BRT (i.e., ``safe'' as per the learned value function), then starting from this state, the corresponding safety controller should keep the system trajectory outside the target set at all times (i.e., the state should also be ``safe'' under the prescribed controller).
Otherwise, we cannot ensure safety of this state, and the value function should be corrected such that this state is inside the BRT.
Once all such states have been added to the learned BRT, we have obtained a provably safe approximation of the true BRT.

We show that the computation of the correction bound can be posed as an optimization problem. However, in general, it is challenging to tractably compute this bound, since the learned value function can be highly nonlinear and hard to optimize.
We propose a scenario-based optimization method to compute this correction. 
Scenario optimization is a sampling-based method to solve semi-infinite optimization problems, and has been widely used for system and control design \cite{ghosh2019new, campi2011sampling, calafiore2006scenario, campi2009scenario}.
The proposed method is not restricted to a specific class of
system dynamics or value functions. 
Given that the value function correction is obtained via randomization and, hence,
is a random quantity, we provide probabilistic guarantees on the safety of the recovered BRT.
However, this confidence is a design parameter and can be chosen as close to 1 as desired (within a simulation budget).

To summarize, the key contributions of this paper are:
\begin{itemize}
    \item an error correction mechanism for DeepReach that results in a provably safe BRT and safety controller for general dynamical systems;
    \item a practical method to compute a probabilistic bound on this error correction that is not restricted to a specific class of systems, resulting in a tractable computation of probabilistically safe reachable tubes; and
    \item a demonstration of the proposed approach for various dynamical systems, inspired by rocket landing and multi-vehicle collision avoidance problems.
\end{itemize}
\section{Problem Setup} \label{sec:problem_setup}
Consider a dynamical system with state $\state \in X \subseteq \R^n$, control $\ctrl \in \cset$, and dynamics $\dot{\state} = \dyn(\state, \ctrl)$ governing how $\state$ evolves over time until a final time $\horizon$. 
Let $\trajstandard(\tdummy)$ denote the state achieved at time $\tdummy \in [\tvar, \horizon]$ by starting at initial state $\state$ and time $\tvar$ and applying control $\cfunc$ over $[\tvar,\tau]$. 
Let $\targetset$ represent a target set of states that the agent wants to either reach (e.g. goal states) or entirely avoid (e.g. obstacles).
\vspace{0.2em}

\noindent \textit{\textbf{Running example: Dubins3D Avoid.}} As a running example, consider a simple low-dimensional system in the literature known as Dubins3D. It involves a car with position $(p_x, p_y)$, heading $\theta$, velocity $v$, and steering control $u_1 \in [u_{\min}, u_{\max}]$. The car's state $x$ evolves according to
%
\begin{align*}
    \dot{p_x} = v\cos{\theta}, \quad \dot{p_y} = v\sin{\theta}, \quad
    \dot{\theta} = u_1
\end{align*}
It wants to avoid a circle with radius $R$ centered at the origin. Thus, we define the target set $\targetset$ of this system to be:
\begin{align*}
    \targetset = \{x: \sqrt{p_x^2 + p_y^2} \le R\}
\end{align*}
We use Dubins3D as a benchmark because it is tractable for traditional reachability methods to compute a solution for, and we can thus compare our results with a ground truth. 
In Section \ref{sec:examples}, we present high-dimensional problems which traditional methods struggle with.

In this setting, we are interested in computing the system's initial-time Backward Reachable Tube, which we denote as $\BRT$. 
We define $\BRT$ as the set of all initial states in $X$ from which the agent will eventually reach $\targetset$ within the time horizon $[0,\horizon]$, despite best control efforts:
\begin{equation}
\label{eq:avoidBRT}
\begin{aligned}
\small{
\BRT = \{\state: \state\in X, \forall \ctrl(\cdot), \exists \tdummy \in [0, \horizon], \trajinittime(\tdummy) \in \targetset\}}
\end{aligned}
\end{equation}
When $\targetset$ represents unsafe states for the system, staying outside of $\BRT$ is desirable. 
When $\targetset$ instead represents the states that the agent wants to reach (e.g., a goal set), $\BRT$ is defined as the set of all initial states in $X$ from which the agent, acting optimally, can eventually reach $\targetset$ within $[0,\horizon]$. Thus, staying within $\BRT$ is desirable.

Our goal in this work is to compute a provable approximation of the safe set. Specifically, we want to compute an approximation $\hat{\safeset}$ such that $\hat{\safeset}\subseteq \BRT^C$ ($\hat{\safeset}\subseteq \BRT$ when $\targetset$ represents goal states). 
We are especially interested in settings where the system is high-dimensional, with which current state-of-the-art reachability methods struggle.

\section{Background: Hamilton-Jacobi (HJ) Reachability and DeepReach} \label{sec:hj_reachability}
In this work, we build upon Hamilton-Jacobi reachability analysis to compute an approximation of the safe set. Here, we provide a quick overview of Hamilton-Jacobi reachability analysis and a specific toolbox to compute high-dimensional reachable sets, DeepReach.

\subsection{Hamilton-Jacobi (HJ) Reachability.} 
In HJ reachabilty, computing $\BRT$ is formulated as an optimal control problem. We will explain it in the context of $\targetset$ being a set of undesirable states. In the end, we will comment on the case where $\targetset$ is a set of desirable states. 

To compute BRT, we first define a target function $\targetfunc(\state)$ such that the sub-zero level of $\targetfunc(\state)$ yields $\targetset$:
\vspace{-0.25em}
\begin{equation}
\small{
    \targetset = \{\state: \targetfunc(\state) \le 0\}}
    \vspace{-0.25em}
\end{equation}
$\targetfunc(\state)$ is commonly a signed distance function to $\targetset$. For example, we can choose $\targetfunc(\state) = \sqrt{p_x^2+p_y^2}-R$ for our Dubins3D Avoid running example.

Next, we define the cost function of a state corresponding to some policy $\cfunc$ to be the minimum of $\targetfunc(\state)$ over its trajectory:
\vspace{-0.75em}
\begin{equation}
    \label{eq:costfunctional}
    \small{
    \costfunctional_{\cfunc}(\state,\tvar) = \min_{\tdummy \in [\tvar, \horizon]} \targetfunc(\trajstandard(\tdummy)).}
    \vspace{-0.25em}
\end{equation}
\normalsize
Since the system wants to avoid $\targetset$, our goal is to maximize $\costfunctional_{\cfunc}(\state,\tvar)$.  
Thus, the value function corresponding to this optimal control problem is:
\vspace{-0.25em}
\begin{equation}
\small{
    \label{eq:valuefunc}
    \vfunc(\state,\tvar) = \sup_{\cfunc} \costfunctional_{\cfunc}(\state,\tvar).}
    \vspace{-0.25em}
\end{equation}
\normalsize

By defining our optimal control problem in this way, we can recover $\BRT$ using the value function. In particular, the value function being sub-zero implies that the target function is sub-zero somewhere along the optimal trajectory, or in other words, that the system has reached $\targetset$. Thus, $\BRT$ is given as the sub-zero level set of the value function at the initial time:

\small
\vspace{-1.5em}
\begin{equation}
    \label{eq:BRT_from_valfunc}
    \BRT = \{\state: \state\in X, \vfunc(\state,0) \le 0 \}
\end{equation}
\vspace{-0.25em}
\normalsize
The value function in Equation \eqref{eq:valuefunc} can be computed using dynamic programming, resulting in the following final value Hamilton-Jacobi-Bellman Variational Inequality (HJB-VI):

%
\small
\vspace{-0.5em}
\begin{equation}
\begin{aligned}
    \label{eq:HJBVI}
    \min\Big\{D_\tvar \vfunc(\state,\tvar)+ H(\state,\tvar), \targetfunc(\state)-\vfunc(\state,\tvar)\Big\} = 0,
    \end{aligned}
    \vspace{-0.25em}
\end{equation}
\normalsize
with the terminal value function $\vfunc(\state,\horizon) = \targetfunc(\state)$. 
$D_\tvar$ and $\nabla$ represent the time and spatial gradients of the value function. 
$H$ is the Hamiltonian that encodes the role of dynamics and the optimal control.

\small
\vspace{-0.5em}
\begin{equation}
    \label{eq:ham}
    \begin{aligned}
    H(\state,\tvar) = \max_\ctrl & \langle \nabla \vfunc(\state,\tvar), \dyn(\state,\ctrl)\rangle.
        \end{aligned}
    \vspace{-0.25em}
\end{equation}
\normalsize

The value function in Equation \eqref{eq:valuefunc} induces the optimal safety controller:

\small
\vspace{-0.5em}
\begin{equation}
    \label{eq:opt_ctrl}
    \begin{aligned}
    u^*(\state,\tvar) = \underset{\ctrl}{\arg\max} & \langle \nabla \vfunc(\state,\tvar), \dyn(\state,\ctrl)\rangle.
        \end{aligned}
    \vspace{-0.25em}
\end{equation}
\normalsize
Intuitively, the safety controller aligns the system dynamics in the direction of the value function's gradient, thus steering the system towards higher-value states.

We have just explained the case where $\targetset$ represents a set of undesirable states. When the system instead wants to reach $\targetset$, an infimum is used instead of a supremum in Equation $\eqref{eq:valuefunc}$.
The control wants to reach $\targetset$, hence there is a minimum instead of a maximum in Equation \eqref{eq:ham} and Equation \eqref{eq:opt_ctrl}.

Traditionally, the value function is computed by solving the HJB-VI over a discretized grid in the state space. Unfortunately, doing so involves computation whose memory and time complexity scales exponentially with respect to the system dimensionality, making these methods practically intractable for high-dimensional systems, such as those beyond 5D. Fortunately, a recent deep learning approach, \metName{}, has been proposed to enable Hamilton-Jacobi reachability for high-dimensional systems. 

\subsection{\metName{} Approximate Solutions.} 
Instead of solving the HJB-VI over a grid, DeepReach represents the value function as a sinusoidal deep neural network (DNN) to learn a parameterized approximation of the value function \cite{bansal2021deepreach}. Thus, memory and complexity requirements for training scale with the value function complexity rather than the grid resolution.
To train the DNN, DeepReach uses self-supervision on the HJB-VI itself. 
Ultimately, it takes as input a state $\state$ and time $\tvar$, and it outputs a learned value $\Tilde{\vfunc}(\state,\tvar)$.
$\Tilde{\vfunc}(\state,\tvar)$ also induces a corresponding policy $\Tilde{\policy}(\state,\tvar)$.
We refer interested readers to \cite{bansal2021deepreach} for further details.

The learned $\Tilde{\vfunc}(\state,\tvar)$ from DeepReach can be used to obtain a BRT as in Equation \eqref{eq:BRT_from_valfunc}, but it will only be as accurate as $\Tilde{\vfunc}(\state,\tvar)$. Unfortunately, like any learning method, $\Tilde{\vfunc}(\state,\tvar)$ can deviate substantially from the true $\vfunc(\state,\tvar)$. 

Our goal is to recover the biggest safe set $\hat{\safeset}$ using $\Tilde{\vfunc}(\state,\tvar)$ that we can (probabilistically) guarantee to be fully contained within the true safe set.
Although we work with \metName{} solutions in particular for this problem setup, our proposed approach in Section \ref{sec:approach} can verify any general $\Tilde{\vfunc}(\state,\tvar)$ and $\Tilde{\policy}(\state,\tvar)$, regardless of whether \metName{}, a level-set method, or some other tool is used to obtain them.
 
\section{Approach}\label{sec:approach}
Our approach is to apply a minimal error correction to $\Tilde{\vfunc}(\state,\tvar)$ such that the corrected value function can be used to extract a safe set.
We will explain our approach in the context of $\targetset$ being a set of undesirable states. In the end of each subsection, we will comment on the case where $\targetset$ represents a set of desirable states. In Section \ref{subsec:errormetric}, we propose an error metric for correcting the value function. In Section \ref{subsec:scenopt}, we propose a practical method to compute a probabilistic bound on this error metric through scenario optimization.

\subsection{Error Metric}\label{subsec:errormetric}
We propose a new error metric $\safetyMetric$ for the value function correction. It is defined as the maximum learned value of an empirically unsafe initial state under the induced policy $\Tilde{\policy}$:
\vspace{-0.25em}
\begin{equation}
\label{eq:reachSafetyMetric}
\begin{aligned}
\small{
\safetyMetric \coloneqq \max_{x\in X}\{\Tilde{\vfunc}(\state,0): \costFunction(\state,0) \le 0\}},
\end{aligned}
\end{equation}
where $\costFunction(\state,0)$ is the cost function associated with the trajectory obtained by using the policy $\Tilde{\policy}(\state,\tvar)$ from an initial state $\state$ and initial time $\tvar=0$ until $\horizon$ (see Equation \eqref{eq:costfunctional}).
Here on, we use $\delta$ as a shorthand for $\safetyMetric$ for brevity purposes.

Intuitively, $\delta$ finds the tightest level of the learned value function that separates the states that are safe under the induced policy $\Tilde{\policy}(\state,\tvar)$ from the ones that are not.
Thus, any initial state within the super-$\delta$ level set of $\Tilde{\vfunc}(\state,0)$ is guaranteed to be safe under the (possibly sub-optimal) policy $\Tilde{\policy}(\state,\tvar)$. 
Lemma \ref{lemma:recoveredsafeset} formalizes this claim.
\begin{lemma}
\label{lemma:recoveredsafeset}
If we compute $\Tilde{\safeset}$ as:
\vspace{-0.25em}
\begin{equation}
\label{eq:reachrecoveredsafeset}
\begin{aligned}
\small{
\Tilde{\safeset} = \{\state\in X: \Tilde{\vfunc}(\state,0) > \delta\}}
\end{aligned}
\vspace{-0.25em}
\end{equation}
then $\Tilde{\safeset} \subseteq \BRT^C$.
\end{lemma}

The proof for Lemma \ref{lemma:recoveredsafeset} is in Appendix\ref{appendix:lemma_recoveredsafeset} of the extended version of this article \cite{lin2022generating}.
Intuitively, Lemma \ref{lemma:recoveredsafeset} states that if we could exactly compute $\delta$, then we can recover a $\Tilde{\safeset}$ that is guaranteed to be a subset of the true safe set. 
Furthermore, it is clear that, by definition, $\delta$ is the smallest (uniform) value adjustment required on $\Tilde{\vfunc}(\state,0)$ to guarantee its safety. 
Thus, $\Tilde{\safeset}$ is the largest safe set we can recover from the learned value function by such a uniform error correction procedure.

Unfortunately, computing $\delta$ is a challenging optimization problem, since both the value function and the cost function in the metric definition \eqref{eq:reachSafetyMetric} are typically non-convex functions of $x$. 
In the next section, we propose a scenario-based optimization approach to compute a high-confidence bound on $\delta$ and thereby generate a high-confidence safe set.

\begin{remark}
We assume there exists an unsafe state for the metric definition \eqref{eq:reachSafetyMetric}. Otherwise, define $\safetyMetric$ trivially as $-\infty$.
\end{remark}

So far, we have discussed an error correction metric for the case where $\targetset$ represents a set of undesirable states. 
When $\targetset$ represents a set of desirable states, we use a minimum instead of a maximum and flip the cost inequality in the metric definition \eqref{eq:reachSafetyMetric}.
To recover the safe set, we extract the sub-$\delta$ level set of the value function in Equation \eqref{eq:reachrecoveredsafeset}.

\subsection{\approachName{} Method}\label{subsec:scenopt}
We will compute a high-confidence bound on $\delta$ by utilizing a random sampling procedure referred to as \textit{scenario optimization} in the systems and control design literature \cite{campi2009scenario}. The \approachName{} Method to compute an approximation $\hat{\delta}$ is summarized in Algorithm \ref{alg:scenarioOptimization}.

\setlength{\textfloatsep}{0pt}
\begin{algorithm}
\caption{\approachName{}}
\label{alg:scenarioOptimization}
\begin{algorithmic}[1]
\REQUIRE $X$, $N$, $M$, $\Tilde{\vfunc}(\state,0)$, $\costFunction(\state,0)$
\STATE $\delta_0 \gets -\infty$
\FOR{$i= 0, 1, \ldots, M-1$} 
    \STATE $\mathcal{D}_i \gets$ Sample $N$ states IID from $\{\state:\state\in X, \Tilde{\vfunc}(\state,0) > \delta_i\}$ \label{line:sample}
    \IF{$\exists \state \in \mathcal{D}_i : \costFunction(\state,0) \le 0$}
        \STATE $\delta_{i} \gets \max_{\state \in \mathcal{D}_i}\{\Tilde{\vfunc}(\state,0): \costFunction(\state,0) \le 0\}$ \label{line:assign}
    \ELSE
        \STATE \textbf{break}
    \ENDIF
\ENDFOR
\RETURN $\hat{\delta} := \delta_i$
\end{algorithmic}
\end{algorithm}

At a high-level, Algorithm \ref{alg:scenarioOptimization} computes a converging sequence of $\delta_i$ that approximates $\delta$ via random sampling until we no longer find any safety violations or reach a maximum number of iterations, $M$.
Specifically, at each iteration $i$, we randomly sample $N$ initial states within the super-$\delta_i$ level set of $\Tilde{\vfunc}$ (Line \ref{line:sample}) using rejection sampling and compute the costs of associated trajectories $\costFunction(\state,0)$ under $\Tilde{\policy}(\state,\tvar)$ as the controller (Equation \eqref{eq:costfunctional}). We next compute the maximum learned value among the states that violate safety and use that as the new estimate of $\delta$ (Line \ref{line:assign})), essentially discarding a level region of $\Tilde{\vfunc}(\state, 0)$ that is empirically unsafe under $\Tilde{\policy}$.
Thus, with each iteration, we obtain a tighter approximation of $\delta$ and terminate when no more safety violations are found or a maximum number of iterations is achieved. 

While Algorithm \ref{alg:scenarioOptimization} is simple to understand and execute, scenario optimization provides a formal guarantee for bounding $\delta$ with $\hat{\delta}$ when the algorithm terminates prior to reaching the maximum number of iterations.
Thus, we can use the approximated $\hat{\delta}$ instead of $\delta$ to compute an approximate safe set $\hat{\safeset}$ similar to the one defined in Lemma \ref{lemma:recoveredsafeset}:
\vspace{-0.25em}
\begin{equation}
\label{eq:reachapproximatelyrecoveredsafeset}
\begin{aligned}
\small{
\hat{\safeset} = \{\state: \state\in X, \Tilde{\vfunc}(\state,0) > \hat{\delta}\}}
\end{aligned}
\vspace{-0.25em}
\end{equation}
Crucially, we can make a formal probabilistic safety guarantee for $\hat{\safeset}$. This is summarized in Theorem \ref{theorem:scenarioOptimization} and proven in Appendix\ref{appendix:theorem_scenopt} of the extended version of this article \cite{lin2022generating}.
\begin{theorem}[Scenario Optimization Verification Theorem]
\label{theorem:scenarioOptimization}
Select a violation parameter $\epsilon \in (0, 1)$ and a confidence parameter $\beta \in (0, 1)$. Pick $N$ such that
\vspace{-0.25em}
\begin{equation} \label{eq:prescribedN}
\begin{aligned}
\small{
N \ge \frac{2}{\epsilon}\left(\ln{\frac{1}{\beta}}+1\right)}
\end{aligned}
\vspace{-0.25em}
\end{equation}
Suppose Algorithm \ref{alg:scenarioOptimization} converges, then with probability at least $1-\beta$, the recovered safe set $\hat{\safeset}$ in Equation \eqref{eq:reachapproximatelyrecoveredsafeset} satisfies:
\begin{equation}\label{eq:guarantee}
\begin{aligned}
\small{
\underset{x \in \hat{\safeset}}{\mathbb{P}}(\vfunc(\state,0) \le 0) \le \epsilon}
\end{aligned}
\vspace{0.50em}\rlap{$\qquad \Box$}
\end{equation}
\end{theorem}

In practice, we find that the algorithm converges within 3-4 iterations.
Intuitively, the higher the quality of the value function approximation, the faster the convergence of $\delta_i$.

Disregarding the confidence parameter $\beta$ for a moment, Theorem \ref{theorem:scenarioOptimization} states that the volume of the unsafe states within $\hat{\safeset}$ is smaller than or equal to the prescribed $\epsilon$ value, as long as we sample enough states to satisfy Inequality \eqref{eq:prescribedN} during the computation of $\hat{\delta}$. 
As $\epsilon$ approaches $0$, the number of safety violations in the recovered safe set also approaches $0$. In
turn, the simulation effort grows unbounded since $N$ is inversely
proportional to $\epsilon$.

To interpret the confidence parameter $\beta$, note that $\hat{\safeset}$ is a random variable that depends on a randomly sampled set of initial states. It may be the case that we just happen to draw a poorly representative sample, in which case the $\epsilon$ bound does not hold. $\beta$ controls the probability of this adverse event happening, which regards the correctness of the final guarantee in Equation \eqref{eq:guarantee}. Fortunately, $N$ only grows logarithmically with $\frac{1}{\beta}$, so $\beta$ can be chosen to be an extremely small value such as $10^{-12}$. $1-\beta$ should then be so close to $1$ that it does not have any practical importance.

When $\targetset$ represents a set of desirable states, we initialize $\delta_0$ to be $\infty$ instead of $-\infty$, flip the inequalities, and take a minimum instead of a maximum in Algorithm \ref{alg:scenarioOptimization}. We flip the value inequalities in Equation \eqref{eq:reachapproximatelyrecoveredsafeset} and Equation \eqref{eq:guarantee}.
\vspace{0.2em}
\begin{remark}
Note that when Algorithm \ref{alg:scenarioOptimization} does not converge, scenario optimization can still be used to bound $\delta$. However, the safety guarantees are more involved. We defer a detailed investigation of these guarantees to future work.
\end{remark}

\noindent \textit{\textbf{Running example: Dubins3D Avoid.}} We now demonstrate the \approachName{} method on the Dubins3D Avoid running example introduced in Section \ref{sec:problem_setup}. We choose system parameters $v=0.6 m/s, u_{\min}=-1.1 rad/s, u_{\max}=1.1 rad/s, R=0.25m$, and use DeepReach to learn the system's value function. 
The used neural network parameters and architecture are the same as in the DeepReach paper \cite{bansal2021deepreach}, i.e., a DNN with 3 hidden layers and 512 neurons in each layer.

For this first example, we purposefully sabotage the training of DeepReach by permitting it to train for only 8,000 epochs, far less than $100$K epochs used in the DeepReach paper. 
We do this to illustrate our method's utility even with low-quality learned solutions. 
We apply our method to a well-trained DeepReach solution in Section \ref{sec:dubins3davoid}.

The overall training took less than 30 minutes on an NVIDIA GeForce RTX 3090Ti.
We execute our verification approach choosing $\beta = 10^{-16}, \epsilon = 10^{-3}$, resulting in an $N = 75683$ as per Theorem 1.
%
%
Thus, we will be $1-10^{-16}$ confident that at least $1-10^{-3}$ of our recovered set will be provably safe.
Algorithm \ref{alg:scenarioOptimization} converges in 4 iterations in this case, resulting in a $\hat{\delta} = 0.3728$.
Slices of the trained BRT (yellow), the recovered BRT (blue), and the ground truth BRT boundary (black) are shown in Figure \ref{fig:baddubinsavoid}. 
The ground truth is computed by a state-of-the-art PDE solver, \textit{Level Set Toolbox (LST)} \cite{LSToolbox}, which computes the value function over a discrete grid of size $101\times 101\times 101$. 

\begin{figure}[ht]
    \vspace{-0.65em}
    \centering
    \includegraphics[width=1 \columnwidth]{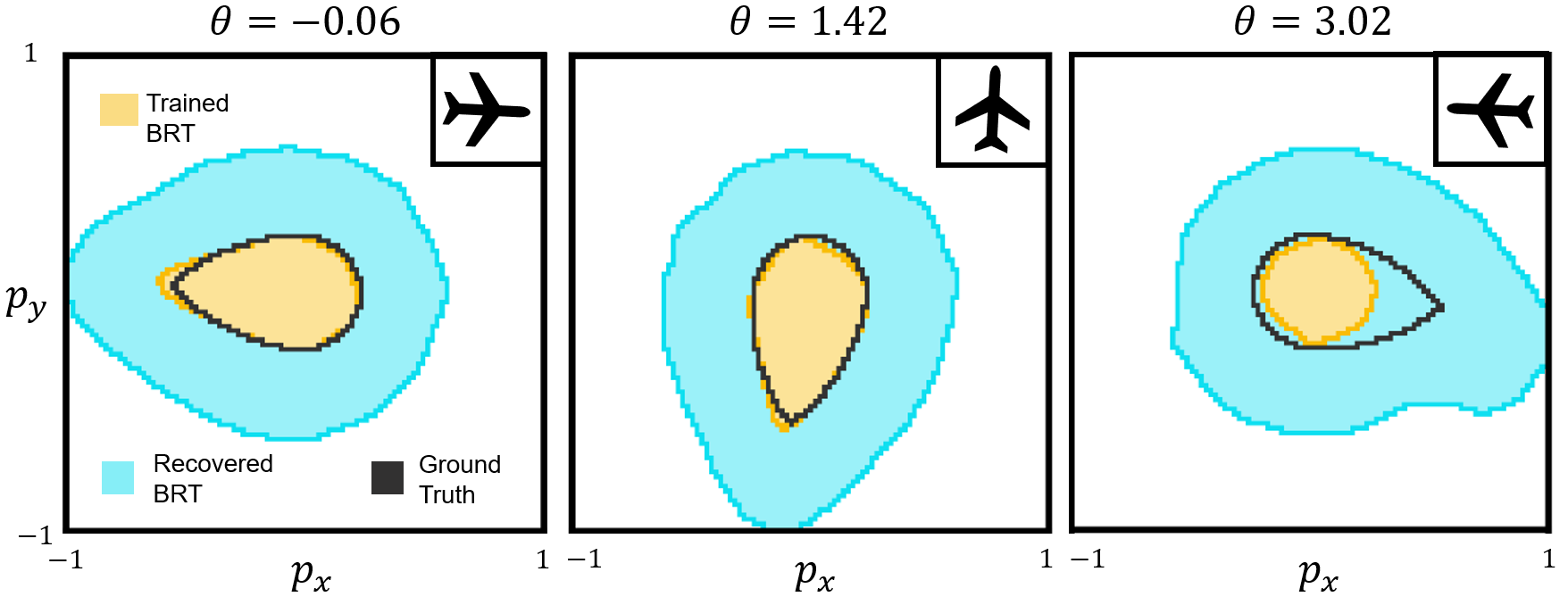}
    \vspace{-1.5 em}
    \caption{Dubins3D Avoid: Slices of the trained, recovered, and ground truth border BRTs for three values of $\theta$ from a sabotaged DeepReach solution. The recovered BRT completely encompasses the ground truth BRT and provides a probabilistically safe approximation of the safe set.}
    \vspace{-0.65em}
    \label{fig:baddubinsavoid}
\end{figure}

As evident from the figure, the recovered BRT slices completely encompass the true unsafe states (the blue region encompasses the states within the black boundary). 
Hence, the complement of the recovered BRT is safe, as we expect based on Theorem \ref{theorem:scenarioOptimization}. 
In the rightmost slice, observe that the trained BRT is smaller than the ground truth BRT and is thus unsafe as it is. 
Our verification approach correctly expands the trained BRT past the ground truth boundary, ensuring safety. 
Validation by sampling 1M states within the recovered safe set reveals a violation rate of $3\times 10^{-6} << \epsilon$.

Note that the trained solution's safety violations are concentrated near $\theta=\pi$. 
It is likely these errors which are responsible for the large expansion of the recovered BRT even for slices that are almost accurate (left and middle slices in Figure \ref{fig:baddubinsavoid}).
This is because our method expands the BRT uniformly.
However, too much expansion is undesirable, since we want to recover as much of the safe set as possible. 

One way to overcome this challenge is to apply refined error correction to $\Tilde{\vfunc}(\state,\tvar)$ by independently verifying separate regions of the states-space (i.e., separately consider each region as our $X$).
This allows us to recover sets with the same safety guarantee but expanded differently depending on their own region's error. 
However, this method's efficacy depends heavily on the choice of sub-regions, which are hard to prespecify, especially for high-dimensional systems.
Instead, we can resort to a data-driven approach.

Specifically, we trained a simple multilayer perceptron (MLP) on 1M randomly sampled initial states and their empirical safety violation costs (the training took an hour on a standard GPU). 
Even though the predictor may be inaccurate (after all, predicting the safe set is the original challenging problem), we hypothesize that it will have learned something useful enough about the general distribution of errors throughout the state space. 
We divided the output of the MLP in 10 bins, each corresponding to a different range of empirical cost.
Thus, we expect the last bin to correspond to the states with maximal safety violations (in this case, that will correspond to states with $\theta=\pi$).
We next ran Algorithm \ref{alg:scenarioOptimization} independently for each bin to compute an error correction.
The corresponding results are shown in Figure \ref{fig:binneddubinsavoid}.
As evident, we get a much tighter approximation of the safe set near the slices that have small errors, allowing us to recover a much bigger safe set overall.
In this case, the recovered safe set volume increased by 35\%. 

%
\begin{figure}[ht]
    \vspace{-0.65em}
    \centering
    \includegraphics[width=1 \columnwidth]{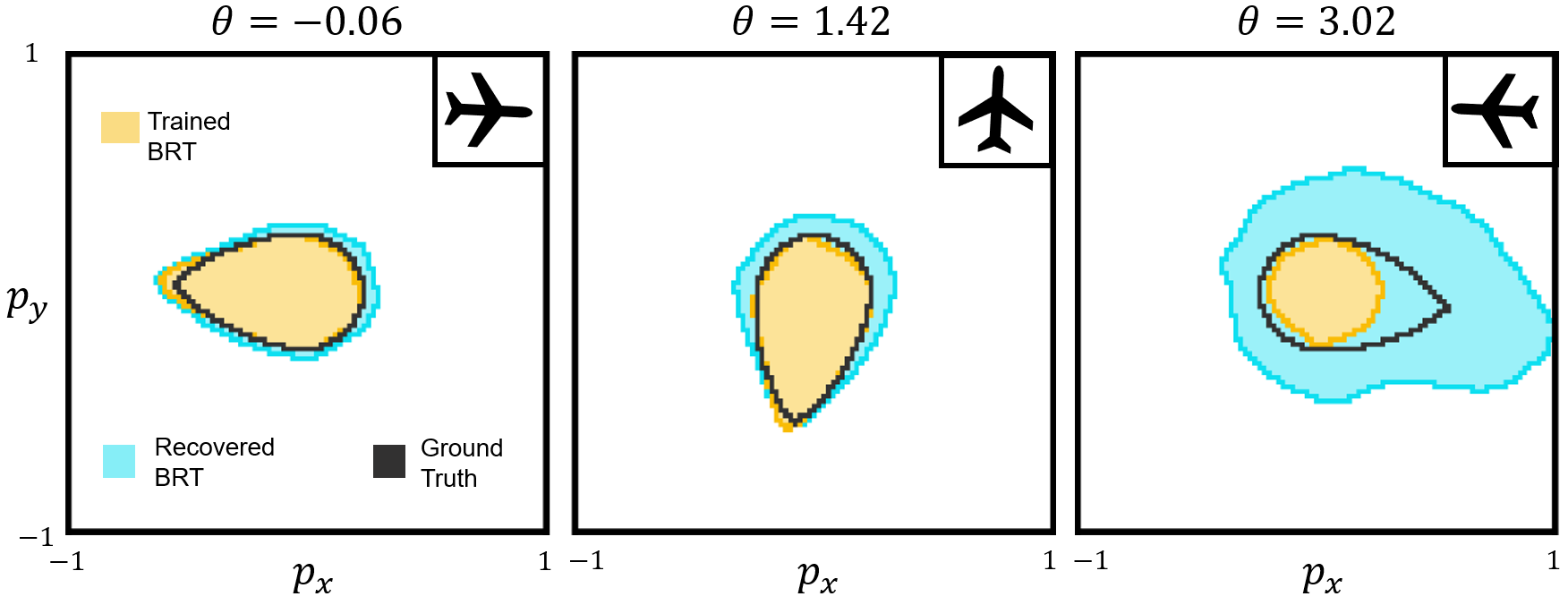}
    \vspace{-1.5em}
    \caption{Dubins3D Avoid: Slices of the trained, recovered, and ground truth border BRTs for three values of $\theta$ from a sabotaged DeepReach solution when recovery is refined by MLP binning. The recovered BRT completely encompasses the ground truth BRT and provides a probabilistically safe approximation of the safe set that is larger than when binning is not used.}
    \vspace{-0.65em}
    \label{fig:binneddubinsavoid}
\end{figure}

As discussed earlier, the performance improvement of this approach depends heavily on the selection of subregions.
We defer this to future work, and for the rest of this paper, we only focus on uniform value function correction.
\section{Case Studies} \label{sec:examples}
We will evaluate our approach on various reachability problems.
First, we show results for the avoid version and the reach version of the low-dimensional Dubins3D system for which we have a ground truth BRT available for comparison. 
Next, we recover safe sets for a 9D multivehicle collision avoidance problem and for a 6D rocket landing problem, with which traditional methods struggle.
Like our running example, we use $\beta = 10^{-16}, \epsilon = 10^{-3}, N = 75683$.
%

\subsection{Dubins3D Avoid} \label{sec:dubins3davoid}
In Figure \ref{fig:dubinsavoid}, we show the results of applying our method to a well-trained DeepReach solution for the same Dubins3D Avoid system as in the running example in Section \ref{sec:approach}.
The solution is trained to consider the periodicity of $\theta$ and for 100K, instead of 8K, epochs.
The proposed algorithm results in a very small $\hat{\delta} = -0.0016$. Interestingly, the fact that $\hat{\delta} < 0$ indicates that the learned value function is conservative. Our method shrinks the trained BRT to recover a larger safe set.
Validation by sampling 1M states within the recovered safe set reveals a violation rate of $3.2\times 10^{-5} << \epsilon$. 

The fact that a much smaller error correction is found for the trained solution shown in Figure \ref{fig:dubinsavoid} than the one in Figure \ref{fig:baddubinsavoid} indicates it is of a much higher quality.
This illustrates how the error correction metric $\delta$ can also be used to evaluate the quality of different approximate value function solutions.
For example, we can evaluate the relative performance of different DNN hyperparameters, which is especially helpful when the ground truth BRT is not available for comparison.  

\begin{figure}[ht]
    \vspace{-0.65em}
    \centering
    \includegraphics[width=1 \columnwidth]{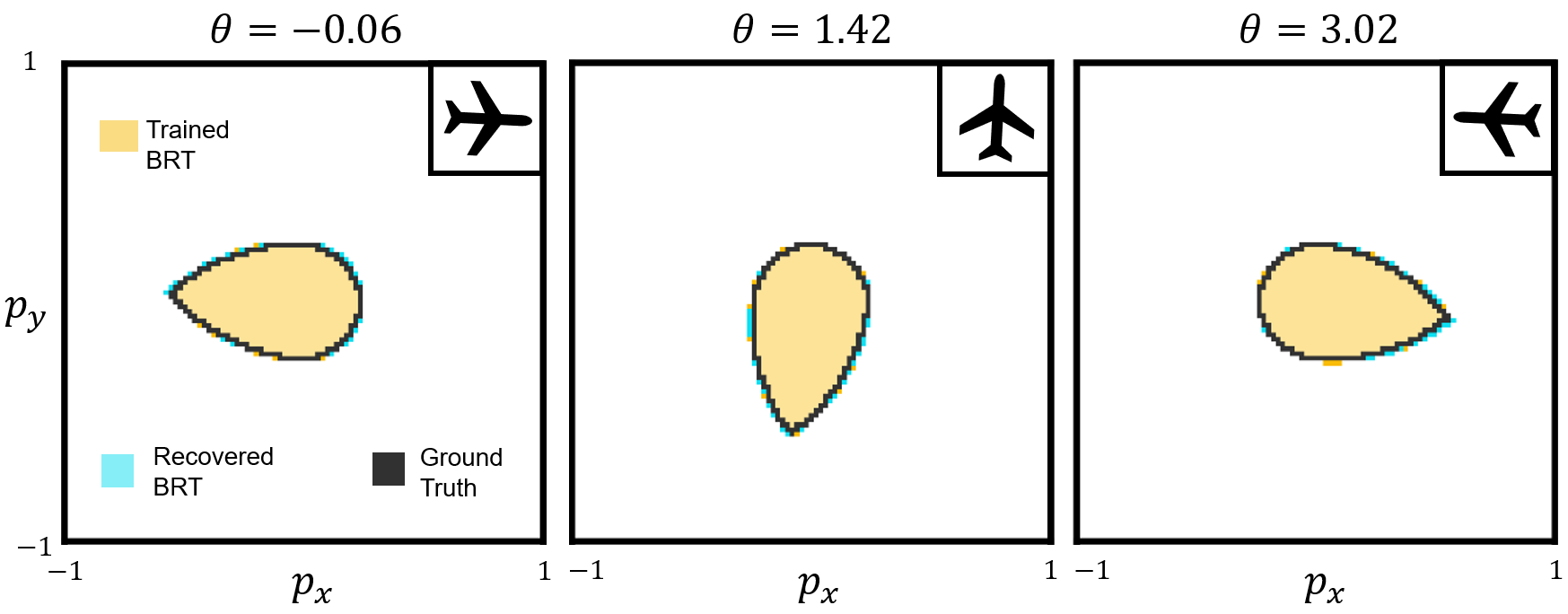}
    \vspace{-1.5em}
    \caption{Dubins3D Avoid: Slices of the trained, recovered, and ground truth border BRTs for three values of $\theta$ from a well-trained DeepReach solution. The trained BRT completely encompasses the recovered BRT, so we plot the recovered BRT border on top for visualization purposes.}
    \vspace{-0.65em}
    \label{fig:dubinsavoid}
\end{figure}

\subsection{Dubins3D Reach} \label{sec:dubins3dreach}
In Figure \ref{fig:dubinsreach}, we show the results of applying our method to a DeepReach solution trained with the same scheme as for the Dubins3D Avoid solution in Section \ref{sec:dubins3davoid}, but for the reach version. 
That is, $\targetset$ is now a set of desirable states, so a safe approximate BRT should be fully contained by the true BRT. 
The trained BRT is just barely crossing the ground truth boundary, and the proposed method cuts it down to just behind the boundary, ensuring system safety.
\begin{figure}[ht]
    \vspace{-0.65em}
    \centering
    \includegraphics[width=1 \columnwidth]{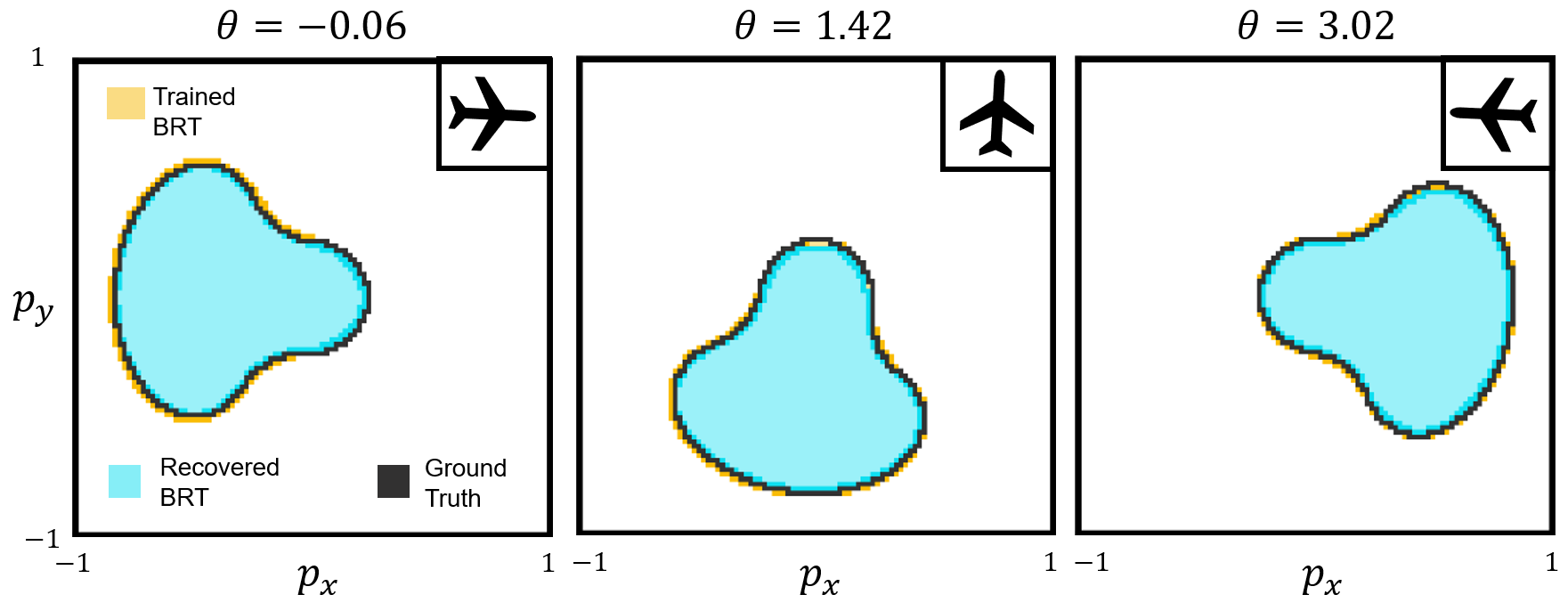}
    \vspace{-1.5em}
    \caption{Dubins3D Reach: Slices of the trained, recovered, and ground truth border BRTs for three values of $\theta$ from a well-trained DeepReach solution. The recovered BRT is correctly a subset of the ground truth BRT.}
    \vspace{-0.65em}
    \label{fig:dubinsreach}
\end{figure}

\subsection{Multivehicle Collision Avoidance} \label{sec:multicollision}
We now consider a 9D collision avoidance system involving 3 independent Dubins3D cars. 
The $i$th car has position $(p_{xi}, p_{yi})$, heading $\theta_i$, velocity $v$, and steering control $u_i \in [u_{\min}, u_{\max}]$. 
The dynamics of vehicle $i$ are given as:
\begin{align*}
    \dot{p}_{xi} = v\cos{\theta_i}, \quad \dot{p}_{yi} = v\sin{\theta_i}, \quad
    \dot{\theta_i} = u_i
\end{align*}
The (undesirable) target set is given by the states where any of the vehicle pairs is in collision:
\begin{equation*}
    \targetset = \{x: \min\{d(\veh_1, \veh_2), d(\veh_1, \veh_3), d(\veh_2, \veh_3)\} \le R\}
\end{equation*}
where $d(\veh_i, \veh_j)$ is the distance between cars $i, j$.
We choose $v=0.6, u_{\min}=-1.1, u_{\max}=1.1, R=0.25$ for our case study.
%
%
The results of applying our method to a DeepReach solution are shown in Figure \ref{fig:multi}.

%
\begin{figure}[ht]
    \vspace{-0.65em}
    \centering
    \includegraphics[width=1 \columnwidth]{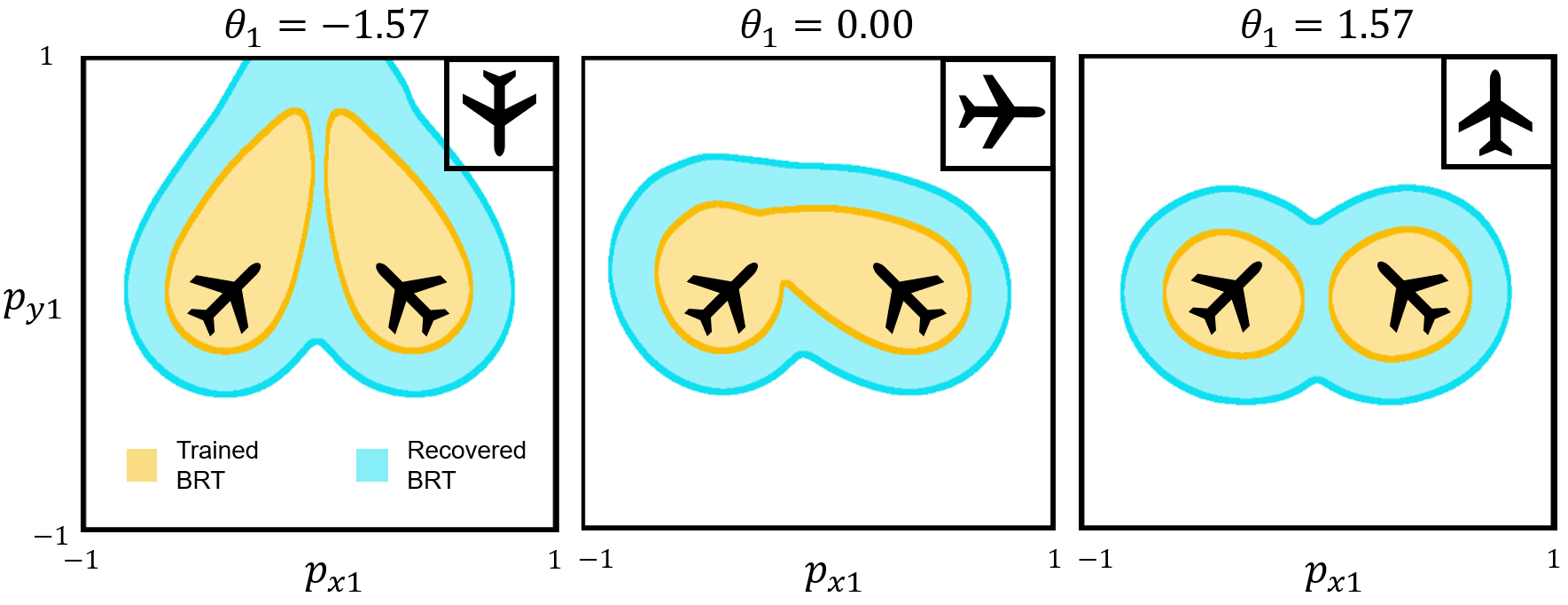}
    \vspace{-1.5em}
    \caption{Multivehicle Collision Avoidance: Slices of the trained and recovered BRTs for three values of $\theta_1$ from a DeepReach solution.}
    \vspace{-0.65em}
    \label{fig:multi}
\end{figure}

This high-dimensional example is typically difficult to compute with traditional methods, yet here we demonstrate successful recovery of a safe set with formal safety guarantees.
While the recovery preserves much of the safe set that is learned, a significant amount is pruned off. 
The percent reduction in safe set size (calculated from 1M samples) is much larger in this example (26\%) than for the Dubins3D solutions in Section \ref{sec:dubins3davoid} (0\%) and Section \ref{sec:dubins3dreach} (10\%), suggesting that the learned value function is more inaccurate for higher dimensional systems. 

To validate safety of the recovered BRT, in Figure \ref{fig:multihist}, we plot the minimum pairwise distance between vehicles along trajectories 
spawning from states sampled within the trained and recovered safe sets. 
Note that the trained safe set contains states which result in a collision, as shown by the mass to the left of the dotted line representing the collision distance. After verification, the recovered safe set shows no such mass. 
Validation by sampling 1M samples in the recovered safe set reveals a violation rate of $5\times 10^{-6} << \epsilon$.

\begin{figure}[ht]
    \vspace{0.5em}
    \centering
    \includegraphics[width=1 \columnwidth]{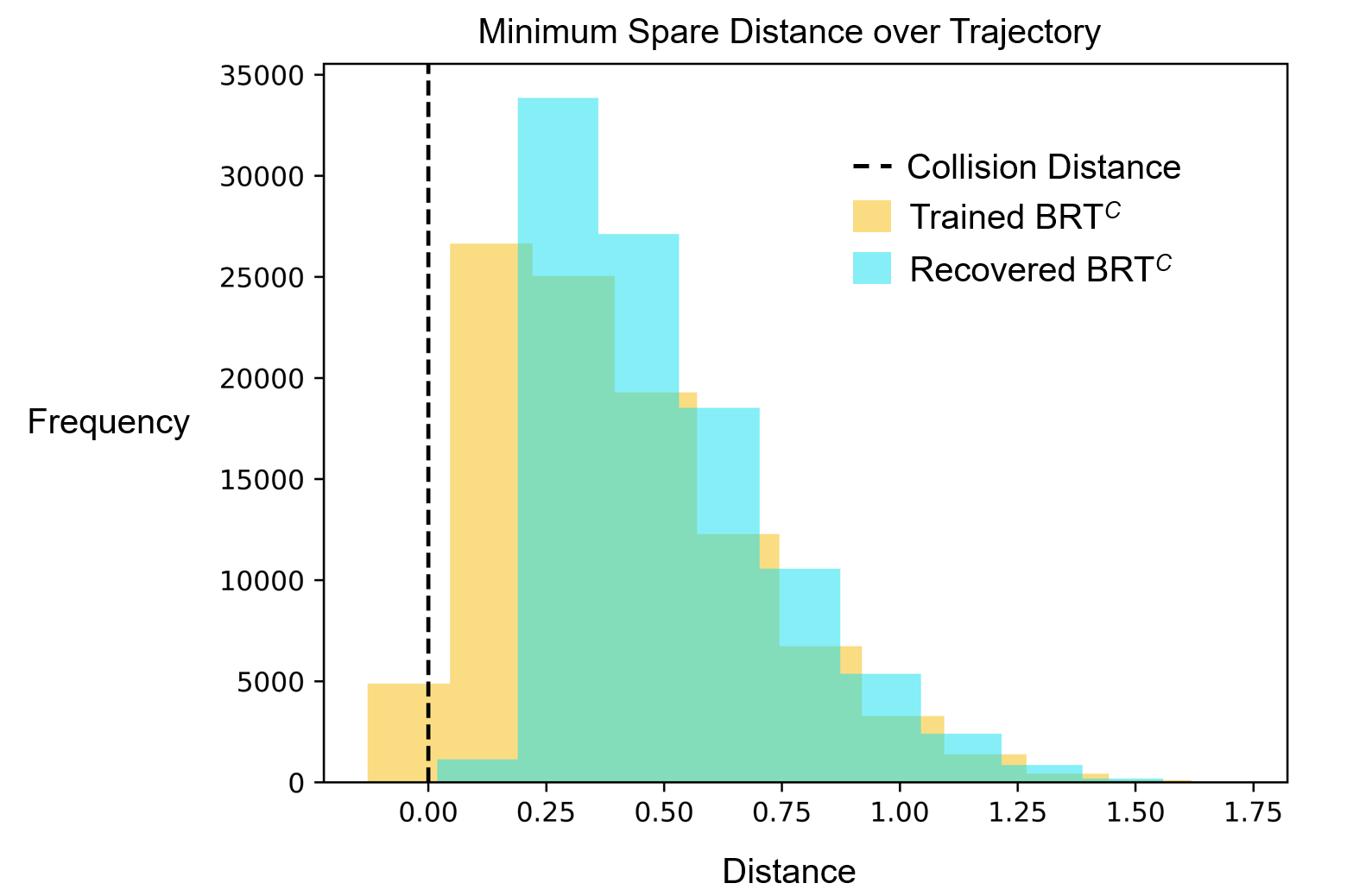}
    \vspace{-1.5em}
    \caption{Multivehicle Collision Avoidance: Histogram of the minimum pairwise spare distance between vehicles along trajectories which spawn from states sampled within the trained and recovered safe sets. The collision distance is indicated by the dotted line.}
    \vspace{0.25em}
    \label{fig:multihist}
\end{figure}

\subsection{Rocket Landing} \label{sec:rocketlanding}
Consider a 6D rocket landing system with position $(p_{x}, p_{y})$, heading $\theta$, velocity $(v_{x}, v_{y})$, angular velocity $\omega$, and torque controls $\tau_1, \tau_2\in [-250, 250]$. 

\noindent The dynamics are:
\begin{gather*}
    \dot{p_x} = v_x,~\dot{p_y} = v_y, ~\dot{\theta} = \omega,~\dot{\omega} = 0.3\tau_1,\\ \dot{v_x} = \tau_1 \cos{\theta} - \tau_2 \sin{\theta},~
    \dot{v_y} = \tau_1 \sin{\theta} + \tau_2 \cos{\theta} - g, 
\end{gather*}
where $g = 9.81$ is acceleration due to gravity.
The target set is the set of states where the rocket reaches a rectangular landing zone of side length 20 centered at the origin:
\begin{align*}
    \targetset = \{\state: |p_x| < 20.0, p_y < 20.0\}
\end{align*}

In Figure \ref{fig:rocket}, we show the results of applying our method to a DeepReach solution. 
We also plot the trajectories emanating from 20 randomly sampled initial states within the recovered safe set, demonstrating that they do indeed reach the target set (the green region) by following the induced policy. Validation by sampling 1M states within the recovered safe set reveals a violation rate of $5\times 10^{-6} << \epsilon$.

\begin{figure}[ht]
    \vspace{-0.65em}
    \centering
    \includegraphics[width=1 \columnwidth]{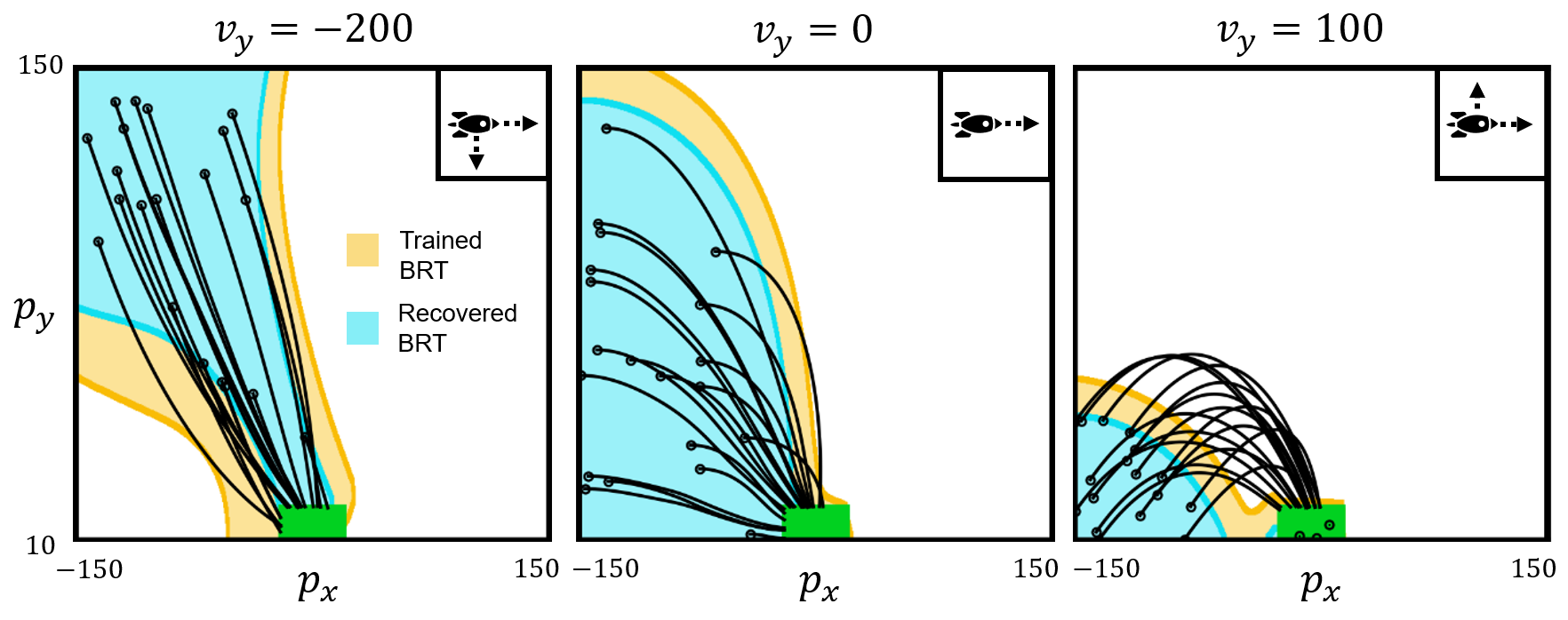}
    \vspace{-1.5em}
    \caption{Rocket Landing: Slices of the trained and recovered BRTs for three values of $v_y$ from a DeepReach solution. The system trajectories safely reach the target set (green) from the initial states within the recovered BRT.}
    \vspace{-0.65em}
    \label{fig:rocket}
\end{figure}
\section{Discussion and Future Work}
In this work, we present an approach to compute an error bound for DeepReach solutions to recover a provably safe approximation of the true reachable tube. We also propose a practical method to compute a probabilistic bound on this error correction that is not restricted to a specific class of systems.
This allows us to utilize the power of learning-based reachability methods to provide probabilistic safety assurances for high-dimensional dynamical systems.
We apply our method to obtain probabilistically safe reachable tubes for high-dimensional rocket-landing and multi-vehicle collision-avoidance problems which traditional methods struggle with. 

In the future, we will explore a more refined approach to error correction (as opposed to a uniform correction), as discussed briefly in the end of Section \ref{sec:approach}. 
Other directions include considering worst-case disturbances in the system dynamics and using the verification approaches proposed here for a targeted refinement of DeepReach solutions.

\appendices

%
%
\clearpage
\newpage
\bibliographystyle{IEEEtran}
\bibliography{initial_sub/Bib/reachability, initial_sub/Bib/bansal_papers, initial_sub/Bib/references,
initial_sub/Bib/opt_ctrl_and_dp}

\newpage
\section*{Appendix}
\subsection{Proof of Lemma \ref{lemma:recoveredsafeset}}\label{appendix:lemma_recoveredsafeset}
\begin{proof}
Assume for a contradiction that we compute $\Tilde{\safeset}$ as in Equation \eqref{eq:reachrecoveredsafeset} and that
\vspace{-0.25em}
\begin{equation}
\begin{aligned}
\small{
\Tilde{\safeset} \not\subseteq \BRT^C}
\end{aligned}
\vspace{-0.25em}
\end{equation}
which implies
\vspace{-0.25em}
\begin{equation}
\begin{aligned}
\small{
\exists \state^{'}: \state^{'} \in \Tilde{\safeset}, \state^{'} \not\in \BRT^C}
\end{aligned}
\vspace{-0.25em}
\end{equation}
$\state^{'} \in \Tilde{\safeset}$ implies that 
\vspace{-0.25em}
\begin{equation}
\begin{aligned}
\small{
\Tilde{V}(\state^{'}, 0) > \delta}
\end{aligned}
\vspace{-0.25em}
\end{equation}
by Equation \eqref{eq:reachrecoveredsafeset}. 

On the other hand, $\state^{'} \not\in \BRT^C$ implies that $V(\state^{'}, 0) \le 0$. 
We know that $\forall (\state,\tvar), \costFunction(\state,\tvar) \le V(\state, \tvar)$ from Equation \eqref{eq:valuefunc}. 
Thus, $\costFunction(\state^{'},0) \le V(\state^{'}, 0) \le 0$. 
Then 
\vspace{-0.25em}
\begin{equation}
\begin{aligned}
\small{
\Tilde{V}(\state^{'}, 0) \le \delta}
\end{aligned}
\vspace{-0.25em}
\end{equation}
by Equation \eqref{eq:reachSafetyMetric}.

We have found the contradiction that $\Tilde{V}(\state^{'}, 0) > \delta$ and $\Tilde{V}(\state^{'}, 0) \le \delta$.
Thus, our assumption must be false, and it must be true that if we compute $\Tilde{\safeset}$ as in Equation \eqref{eq:reachrecoveredsafeset}, then $\Tilde{\safeset} \subseteq \BRT^C$.

To prove the equivalent of Lemma \ref{lemma:recoveredsafeset} but in the case when $\targetset$ represents a set of desirable states, we use $\BRT$ instead of $\BRT^C$ and flip the value and cost inequalities.
\end{proof}

\subsection{Lemma \ref{lemma:scenarioOptimization}}\label{sec:lemma_scenarioOptimization}
\begin{lemma}
\label{lemma:scenarioOptimization}
Select a violation parameter $\epsilon \in (0, 1)$ and a confidence parameter $\beta \in (0, 1)$. Pick $N$ such that
\vspace{-0.25em}
\begin{equation}
\begin{aligned}
\small{
N \ge \frac{2}{\epsilon}\left(\ln{\frac{1}{\beta}}+1\right)}
\end{aligned}
\vspace{-0.25em}
\end{equation}
Then, with probability at least $1-\beta$, the solution $\hat{\delta}$ to Algorithm \ref{alg:scenarioOptimization} executed with $N$ satisfies the following condition.
\vspace{-0.25em}
\begin{equation} \label{eq:metricbound}
\begin{aligned}
\small{
\underset{\{\state: \state\in X, \Tilde{\vfunc}(\state,0) > \hat{\delta}\}}{\mathbb{P}}\left(\left(\Tilde{\vfunc}(\state,0): \costFunction(\state,0) \le 0\right) > \hat{\delta}\right) \le \epsilon}
\end{aligned}
\vspace{-0.25em}
\end{equation}
\end{lemma}
\begin{proof}
Lemma \ref{lemma:scenarioOptimization} follows from Theorem 1 in Scenario Optimization \cite{campi2009scenario}. To use the theorem, we need to prove the following conditions:
\begin{enumerate}
    \item computing $\delta$ can be converted into a standard Scenario Optimization problem
    \item Algorithm \ref{alg:scenarioOptimization} obtains $\hat{\delta}$ by sampling i.i.d from the same space that $\delta$ is optimized over, $X$, but with a probability distribution that is uniform over $\{\state: \state\in X, \Tilde{\vfunc}(\state,0) > \hat{\delta}\}$
\end{enumerate}
$\delta$ can be formalized as the following optimization problem:
\begin{align*}
    &\min{g} \\
    &\text{s.t. } \forall \state\in X, \left(\Tilde{\vfunc}(\state,0): \costFunction(\state,0) \le 0\right) \le g
\end{align*}
This is a semi-infinite optimization problem where the constraints are linear, and thus convex, in the optimization variable $g$ for any given $\state$. Thus, Lemma \ref{lemma:scenarioOptimization} follows from Theorem 1 in Scenario Optimization \cite{campi2009scenario} by replacing $c$ by $1$, $\gamma$ by $g$, $\Delta$ by the state space of interest $X$, and $f$ by $\left(\Tilde{\vfunc}(\state,0): \costFunction(\state,0) \le 0\right) - g$. To use Theorem 1, we also require that i.i.d samples are chosen according to the uniform distribution over $\{\state: \state\in X, \Tilde{\vfunc}(\state,\tvar) > \hat{\delta}\}$. This can be proven by observing that, in Algorithm \ref{alg:scenarioOptimization}, the last iteration of the while loop samples $\state$ randomly and independently from $\{\state: \state\in X, \Tilde{\vfunc}(\state,\tvar) > \hat{\delta}\}$ where $\hat{\delta}$ is exactly the final metric returned. Importantly, the algorithm does not update $\hat{\delta}$ in the final loop (since the loop subsequently breaks).

To prove the equivalent of Lemma \ref{lemma:scenarioOptimization} but in the case when $\targetset$ represents a set of desirable states, we flip the value and cost inequalities. We also take a maximum instead of a minimum and flip the inequality with $g$ when we formalize $\delta$ as an optimization problem.
\end{proof}

\subsection{Proof of Theorem \ref{theorem:scenarioOptimization}}\label{appendix:theorem_scenopt}
\begin{proof}
We rewrite the LHS of Inequality \eqref{eq:metricbound} of Lemma \ref{lemma:scenarioOptimization} in Appendix\ref{sec:lemma_scenarioOptimization} by substituting in $\hat{\safeset}$ by Equation \eqref{eq:reachapproximatelyrecoveredsafeset}:
\vspace{-0.25em}
\begin{align*}
\underset{\{\state: \state\in X, \Tilde{\vfunc}(\state,0) > \hat{\delta}\}}{\mathbb{P}}&\left(\left(\Tilde{\vfunc}(\state,0): \costFunction(\state,0) \le 0\right) > \hat{\delta}\right) \\
= \underset{\{\state: \state\in X, \Tilde{\vfunc}(\state,0) > \hat{\delta}\}}{\mathbb{P}}&\left(\left(\Tilde{\vfunc}(\state,0) > \hat{\delta}\right) \cap \left(\costFunction(\state,0) \le 0\right)\right) \\
= \underset{\{\state: \state\in X, \Tilde{\vfunc}(\state,0) > \hat{\delta}\}}{\mathbb{P}}&\left(\costFunction(\state,0) \le 0\right) \\
= \underset{\state \in \hat{\safeset}}{\mathbb{P}}&\left(\costFunction(\state,0) \le 0\right)
\end{align*}
\vspace{-0.25em}
where the second-to-last line follows because the probability is over the set of states that are already given to have a learned value greater than $\hat{\delta}$.

We know that $\forall (\state,\tvar), \costFunction(\state,\tvar) \le V(\state, \tvar)$ from Equation \eqref{eq:valuefunc}. Thus, it follows that:
\vspace{-0.25em}
\begin{equation}
\begin{aligned}
\small{
\underset{\state \in \hat{\safeset}}{\mathbb{P}}\left(\vfunc(\state,0) \le 0\right) \le \underset{\state \in \hat{\safeset}}{\mathbb{P}}\left(\costFunction(\state,0) \le 0\right) \le \epsilon}
\end{aligned}
\vspace{-0.25em}
\end{equation}

To prove the equivalent of Theorem \ref{theorem:scenarioOptimization} but in the case when $\targetset$ represents a set of desirable states, we flip the value and cost inequalities.
\end{proof}

\end{document}